\documentclass[sigconf, nonacm]{acmart}

\newcommand\vldbdoi{XX.XX/XXX.XX}

\usepackage{amsmath} 
\usepackage{booktabs} 
\usepackage{array}
\usepackage{graphicx} 
\usepackage{multirow}
\usepackage{color}
\usepackage{xcolor}
\definecolor{Violet}{RGB}{148,0,211} 
\usepackage{makecell} 
\usepackage{float}
\usepackage{subcaption}
\usepackage{colortbl}
\usepackage{placeins}
\usepackage{amsmath}  

\usepackage{amsthm}   
\usepackage{amssymb}  
\newtheorem{thm}{Theorem}  
\newtheorem{proposition}{Proposition}

\usepackage[linesnumbered,vlined,ruled,noend]{algorithm2e}
\usepackage{color, soul}

\usepackage{wrapfig}
\begin{document}
\title{Data Proportion Detection for Optimized Data Management for Large Language Models [Vision]}

\author{Hao Liang $^\dagger$}
\affiliation{%
  \institution{Peking University}
  \city{Beijing}
  \state{China}
}
\email{hao.liang@stu.pku.edu.cn}

\author{Keshi Zhao $^\dagger$}
\affiliation{%
  \institution{Peking University}
  \city{Beijing}
  \state{China}
}
\email{zhaoks@stu.pku.edu.cn}

\author{Yajie Yang}
\affiliation{%
  \institution{Beijing University of Posts and Telecommunications}
  \city{Beijing}
  \state{China}
}
\email{2022211168@bupt.cn}

\author{Bin Cui}
\affiliation{%
  \institution{Peking University}
  \city{Beijing}
  \state{China}
}
\email{bin.cui@pku.edu.cn}

\author{Guosheng Dong}
\affiliation{%
  \institution{Baichuan Inc.}
  \city{Beijing}
  \state{China}
}
\email{dongguosheng@baichuan-inc.com}

\author{Zenan Zhou $^*$}
\affiliation{%
  \institution{Baichuan Inc.}
  \city{Beijing}
  \state{China}
}
\email{zhouzenan@baichuan-inc.com}

\author{Wentao Zhang $^*$}
\affiliation{%
  \institution{Peking University}
  \city{Beijing}
  \state{China}
}
\email{wentao.zhang@pku.edu.cn}





\begin{abstract}
Large language models (LLMs) have demonstrated exceptional performance across a wide range of tasks and domains, with data preparation playing a critical role in achieving these results. Pre-training data typically combines information from multiple domains. To maximize performance when integrating data from various domains, determining the optimal data proportion is essential. However, state-of-the-art (SOTA) LLMs rarely disclose details about their pre-training data, making it difficult for researchers to identify ideal data proportions. In this paper, we introduce a new topic, \textit{data proportion detection}, which enables the automatic estimation of pre-training data proportions by analyzing the generated outputs of LLMs. We provide rigorous theoretical proofs, practical algorithms, and preliminary experimental results for data proportion detection. Based on these findings, we offer valuable insights into the challenges and future directions for effective data proportion detection and data management. The code is available at \url{https://github.com/yangyajie0625/data_detection}.
\end{abstract}

\maketitle
\begingroup
\renewcommand\thefootnote{}\footnote{\noindent
$\dagger$ The first two authors have equal contributions. \\
$*$ Corresponding Author
}
\addtocounter{footnote}{-1}
\endgroup

\section{Introduction}
Large language models (LLMs) have exhibited exceptional performance across a wide range of tasks and domains~\cite{llama,llama3repo,qwen}. However, the rapid development of LLMs has significantly increased the size of pre-training datasets, creating an urgent need for efficient large-scale data management systems~\cite{fernandez2023large, miao2024demystifying, chen2023lingua, prakash2024integrating, wang2023data}. When constructing pre-training datasets, it is essential to integrate data from various sources to achieve optimal performance~\cite{bai2024survey}. Unfortunately, many state-of-the-art models keep their data mixing proportions proprietary, making it difficult for researchers and companies to determine the ideal data mixture. Currently, in preparing data for LLMs, we face the following three challenges:

\textbf{C1. Ineffective Training Results}
Without optimal data proportion, pre-trained models struggle to learn effectively, leading to suboptimal performance in downstream tasks due to insufficient learning of critical features~\cite{xie2024doremi,fan2023doge}.

\textbf{C2. Waste of Computational Resources}
The lack of optimal data proportion forces models to use excessive computation resources, prolonging training time and increasing overall costs~\cite{xie2024doremi,fan2023doge}.

\textbf{C3. Waste of Data and Management Costs}
Suboptimal data proportions reduce the value derived from available data while increasing data management costs, as both data overuse and underuse contribute to unnecessary overhead.

To address these challenges, we introduce a new concept called \textit{data proportion detection}, which focuses on identifying the proportion of pre-training data utilized by a model without prior knowledge of the original pre-training data. In this paper, we provide rigorous theoretical justification for this concept and propose a practical algorithm to estimate the pre-training data proportion by leveraging the generated data proportion. Our findings also offer valuable insights for future research directions on the data proportion detection problem.

This paper presents three key takeaway messages:
\begin{itemize}
    \item \textbf{New Perspective} We propose a new perspective for detecting the data proportions in closed-source pre-trained models, which enables effective management of pre-training data proportions and thus reduces the cost associated with adjusting data proportions during pre-training.
    \item \textbf{New Topic} We propose a data proportion detection algorithm grounded in a rigorous theory of data distribution. Using this algorithm, we conducted a preliminary data proportion detection experiment, which establishes a baseline for data proportion detection and serves as a foundation for future research.
    \item \textbf{New Challenges in Data Management} This paper outlines three key challenges in data management: Fast and Large-Scale LLM Inference Systems, Robust Data Cleaning and Classification Systems, the Next Generation of Data Mixing Laws, and robust of data preparing system.
\end{itemize}

The paper is organized as follows: We give preliminary introduction to LLMs and data preparation for LLMs in section \ref{sec: Preliminary}. We then introduce the data proportion detection problem and gave rigorous proof and proposed a practical algorithm for the data proportion detection problem in section \ref{sec: Problem Formulation}. After that, we gave preliminary experiment data proportion detection results in section \ref{sec: Preliminary Experiments}. With the results, we analyze the challenges and future work need for data proportion detection in section \ref{sec: Challenges}. Finally we summarize the results of this vision paper in section \ref{sec: Conclusion}.

\section{Preliminary}\label{sec: Preliminary}
\subsection{Pre-training Stage in LLMs}
The pre-training stage is a critical phase in the development of large language models (LLMs), where the model learns fundamental linguistic patterns from large-scale corpora. During this stage, the model is exposed to vast amounts of text data from various domains, enabling it to acquire a broad understanding of language. The primary objective is for the model to capture syntactic and semantic relationships between words and phrases, thereby developing a foundational knowledge that can be applied to a range of downstream tasks. This stage typically involves unsupervised learning, where the model attempts to predict tokens in a sequence without task-specific guidance, allowing it to generalize effectively across different contexts.

In the pre-training of large language models (LLMs), the goal is to predict the next token in a sequence by minimizing the negative log-likelihood of the predicted token given its context. The loss function for a sequence \( y_0, y_1, \dots, y_n \) is:

\[
\mathcal{L}^{\mathrm{PT}}(y) = -\sum_{i=1}^n \log P\left(y_i \mid y_0, \dots, y_{i-1}; \boldsymbol{\theta} \right)
\]

This function measures how well the model predicts each token given the preceding tokens, and by minimizing this loss, the model learns to generate contextually appropriate sequences. Attention mechanisms, particularly Transformer architectures, are employed to capture long-range dependencies in the data.


\textbf{Scaling Laws and Data Quality}:  
Scaling Laws describe the relationship between the size of the model, the amount of data, and performance~\cite{kaplan2020scaling, hoffmann2022training}. As the number of parameters in the model increases, or as more data is introduced, the model’s performance improves, but with diminishing returns at a certain point. This creates a challenge where simply increasing model size or dataset size does not guarantee proportionate gains in performance. However, high-quality data can break Scaling Laws by boosting model efficiency beyond what is expected from quantity alone~\cite{sorscher2022beyond}.

One key factor is \textit{data quality}—carefully curated, diverse, and representative datasets lead to more efficient learning. High-quality data allows models to generalize better, even with fewer parameters or less training data. Optimal \textit{data proportioning} across domains also plays a critical role in this process. For example, selecting the right mix of domain-specific and general data ensures that the model is not only learning from the most relevant examples but also retains a broad understanding of language use cases. Hence, focusing on data quality and the optimal ratio of data sources can lead to breakthroughs in model performance that would not be achievable by scaling data size alone~\cite{xie2024doremi,fan2023doge}.


By the end of the pre-training stage, the model has learned general-purpose linguistic knowledge, which can be fine-tuned for specific tasks in later stages, such as question answering, summarization, or translation.

\subsection{Data Domain Proportioning in LLMs}
A pivotal factor in large language model (LLM) pre-training is the distribution of data across various domains, such as news, science, and social media. Achieving the right balance between these domains is crucial for model generalization across a wide range of tasks. For instance, prioritizing technical domains enhances performance on specialized tasks, while a balanced blend fosters general robustness. Misaligned domain proportioning risks overfitting to specific domains, ultimately hindering the model's ability to perform well on diverse tasks~\cite{bai2024survey, wang2023data}. 

State-of-the-art (SOTA) models like Qwen~\cite{qwen} and LLaMA~\cite{llama} leverage extensive computational resources to achieve optimal data domain balance. To mathematically model the relationship between pre-training data proportions and model performance,~\cite{ye2024data} introduced a \textbf{Data Mixing Law}, expressed as:
\[
\underset{y \sim D_i}{\mathbb{E}} [\operatorname{Loss}(y)] = L_{D_i}(\alpha_1, \cdots, \alpha_m) \approx c_i + k_i \exp \left(\sum_{j=1}^M t_{ij} \alpha_j \right).
\]
This formula describes how the expected loss over a domain \( D_i \) is a function of the domain proportions \( \alpha_1, \alpha_2, \dots, \alpha_m \). The coefficients \( c_i, k_i, \) and \( t_{ij} \) represent constants specific to the domain, capturing the interaction between different domain proportions. This model helps guide decisions on how to mix and balance data from various domains to optimize LLM performance.

In conclusion, proper domain balancing during pre-training equips LLMs with both general and specialized knowledge, ensuring adaptability to a wide array of tasks.

\subsection{Data Preparation for LLMs}
The advent of large language models has brought about a substantial increase in the volume of training data.~\cite{llama, openai2023gpt} In this scenario, the quality and quantity of data become paramount. LLMs, trained on vast amounts of data, can capture subtle nuances and complex patterns in language, excelling in various natural language processing tasks. However, the increase in data volume also brings new challenges, particularly in data management, cleaning, and annotation.~\cite{bai2024survey} In this section, we mainly discuss data management systems for data preparation, data quality, and data selection.

\paragraph{Data Preparation Systems} Researchers in our community have made significant contributions to developing data preparation systems. They have explored how pretrained language models can be leveraged for database configuration~\cite{trummer2022db} and for enhancing SQL queries~\cite{trummer2022codexdb}. Recently, there has been growing interest in preparing data for generative large language models. For instance, ~\cite{fernandez2023large} discusses how LLMs are poised to disrupt data management. While data juicer~\cite{chen2024data} utilize multiple operators for managing large amount of data for LLMs and Multimodal LLMs. 

\paragraph{Data Quality}: High-quality data can significantly enhance the performance of models~\cite{llama3repo}. As the volume of data increases, ensuring high data quality becomes more challenging because it requires more resources for data cleaning, selection and annotation.~\cite{bai2024survey} Poor quality data can lead to models learning incorrect patterns and making inaccurate predictions. Data quality includes multiple aspects including optimal data mixing proportion.

\paragraph{Data Cleaning}: 
LLMs-based methods were commonly used in data cleaning for higher quality data~\cite{bai2024survey}. For instance, \citet{du2023mods} leverages DeBERTa~\cite{he2020deberta} for scoring, retaining high-quality data, and combining it with the k-center greedy algorithm to select diverse data. \citet{chen2023alpagasus} score the accuracy of data using ChatGPT to pick out high-quality data. \citet{xu2023rethinking} use GPT-4 to rewrite data to increase their complexity and then streamline it by reducing its variety and improving its quality. \citet{liu2023makes} train two models using ChatGPT's labeled data to score the quality and complexity of the data. \citet{lu2023instag} rely on ChatGPT to tag each instance, defining its complexity and diversity based on these tags. \citet{parkar2024selectllm} first cluster the data, and then use GPT-4 to select high-quality data for each cluster.

In summary, the rapid expansion of training data required for large language models has intensified the importance of efficient data preparation systems, data quality, and selection processes. While numerous methodologies and systems, such as Data Juicer and LLM-based data cleaning approaches, have made significant strides in managing large-scale data, challenges still remain. Ensuring high-quality data in the face of increasing volumes demands innovative tools and techniques that streamline data filtering, annotation, and selection. Continued advancements in both data management algorithms and system infrastructure are essential to fully unlock the potential of LLMs in various applications.


\section{Problem Formulation}\label{sec: Problem Formulation}
In this section, we first introduce the data proportion detection problem in section \ref{sec: Data_Detection}. Then we proved we can use generated data proportion to estimate the pre-training data proportion in section \ref{sec: Theory_Data_Detection}. After the proof, we conduct a practical algorithm for data proportion detection in section \ref{sec: Preliminary Algorithm}.

\subsection{Data Proportion Detection}\label{sec: Data_Detection}
Given a large language model (LLM) trained on data from multiple domains \( \mathcal{D} = \{D_1, D_2, \dots, D_n\} \), where each domain represents a distinct type of data (e.g., web text, mathematical problems, code), the model has been exposed to varying proportions of data from these domains. Let \( \alpha_i \) denote the proportion of data from domain \( D_i \) used during the training process, where \( \sum_{i=1}^{n} \alpha_i = 1 \) and \( \alpha_i \geq 0 \).

The objective of this research is to investigate the reverse problem: Given the trained LLM \( M \), can we accurately estimate the proportions \( \{\alpha_1, \alpha_2, \dots, \alpha_n\} \) of the training data from each domain? Specifically, we aim to develop methodologies that can, using only the LLM and without access to the original training data, infer the relative contributions of each domain to the model's training. This problem is critical for understanding the domain biases inherent in LLMs and for improving transparency and interpretability in model deployment.

Formally, the problem can be defined as follows:

\begin{itemize}
    \item \textbf{Input}: A trained LLM \( M \) with no access to the original training data.
    \item \textbf{Output}: Estimated domain proportions \( \hat{\alpha}_1, \hat{\alpha}_2, \dots, \hat{\alpha}_n \) that approximate the true training data proportions \( \alpha_1, \alpha_2, \dots, \alpha_n \).
\end{itemize}

The challenge lies in the inherent complexity of the LLM's architecture and the potential overlap between domains, which may blur the distinct signals from individual domains. The solution to this problem requires the development of innovative techniques that can analyze the LLM's behavior, responses, or internal representations to infer the domain distribution \( \{\alpha_1, \alpha_2, \dots, \alpha_n\} \).

\subsection{Theory of Data Proportion Detection}\label{sec: Theory_Data_Detection}
In this section, we first connect generation data proportion with loss of LLMs in proposition \ref{pro:1}, then connect loss of LLMs with pre-training data proportion in proposition \ref{pro:2}.
\begin{proposition}[Probability of a Sentence Belonging to a Domain]\label{pro:1}
Let \( y \) be a generated sentence in domain \( D_i \). The probability that \( y \) belongs to domain \( D_i \) can be expressed as:
\[
P(y \in D_i) := \gamma_i = exp\left(-\underset{y \sim D_i}{\mathbb{E}} [\operatorname{Loss}(y)]\right).
\]
\end{proposition}

\begin{proof}
Let \( y \) be a generated sentence in domain \( D_i \). We want to calculate the probability that \( y \) belongs to domain \( D_i \). The probability can be expressed as:

\[
P(y \in D_i) = \underset{y \sim D_i}{\mathbb{E}} [P(y)].
\]

Expanding this expectation, we have:

\[
\begin{aligned}
P(y \in D_i) &= \underset{y \sim D_i}{\mathbb{E}} [P(y)] \\
             &= \underset{y \sim D_i}{\mathbb{E}} [P(y_1, \cdots, y_n)] \\
             &= \underset{y \sim D_i}{\mathbb{E}} \left[P(y_1) \cdots P(y_n \mid y_1, \cdots, y_{n-1})\right].
\end{aligned}
\]

Given that \( P(y_1) = P(\langle \text{bos} \rangle) = 1 \), this simplifies to:

\[
\begin{aligned}
P(y \in D_i) &= \underset{y \sim D_i}{\mathbb{E}} \left[P(y_2 \mid y_1) \cdots P(y_n \mid y_1, \cdots, y_{n-1})\right] \\
             &= \underset{y \sim D_i}{\mathbb{E}} \left[exp\left({\log P(y_2 \mid y_1) + \cdots + \log P(y_n \mid y_1, \cdots, y_{n-1})}\right)\right].
\end{aligned}
\]

Since \( \log P(y_j \mid y_1, \cdots, y_{j-1}) \) is the negative loss for the token \( y_j \), the expression can be rewritten as:

\[
P(y \in D_i) = \underset{y \sim D_i}{\mathbb{E}} \left[exp\left({-\operatorname{Loss}(y)}\right)\right].
\]

By definition, this expectation is equivalent to the exponential Expected Loss of the sentence \( y \) in domain \( D_i \):

\[
P(y \in D_i) = exp\left(-\underset{y \sim D_i}{\mathbb{E}} [\operatorname{Loss}(y)]\right).
\]

\end{proof}

This expected loss \( \underset{y \sim D_i}{\mathbb{E}} [\operatorname{Loss}(y)] \) can be approximated using Monte Carlo sampling as:
\[
\underset{y \sim D_i}{\mathbb{E}} [\operatorname{Loss}(y)] \approx \frac{1}{n} \sum_{j=1}^n \operatorname{Loss}(y_j), \quad y_j \overset{\text{i.i.d.}}{\sim} D_i.
\]

From proportion \ref{pro:1}, we relate the synthetic data proportion with perplexity and Loss of model. By applying the following data mixing law from~\cite{ye2024data}, we can relate the synthetic data proportion with the training data proportion.

\begin{thm}[Data Mixing Law]
For a LLM which is trained by data proportion \( \alpha_1, \alpha_2, \dots, \alpha_n \), the expected loss can be expressed as:
\[
\underset{y \sim D_i}{\mathbb{E}} [\operatorname{Loss}(y)] = L_{D_i}(\alpha_1, \cdots, \alpha_m) \approx c_i + k_i \exp \left(\sum_{j=1}^M t_{ij} \alpha_j \right).
\]
\end{thm}

\begin{proposition}[data proportion detection]\label{pro:2}
The vector of training data proportions \( \vec{\alpha} \) can be calculated using the following equation:
\[
\vec{\alpha} = T^{-1} \vec{\beta},
\]
where \( T = [t_{ij}] \) is the matrix of coefficients, and the vector \( \vec{\beta} \) is defined as:
\[
\beta_i = \log \left(-\frac{\log \gamma_i + c_i}{k_i}\right),
\]
\end{proposition}
\begin{proof}
We begin by considering the relationship between the expected perplexity and the expected loss. Specifically, we have:
\[
\gamma_i = \underset{y \sim D}{\mathbb{E}} \left(exp\left[-\operatorname{Loss}(y)\right]\right) = exp\left(-\underset{y \sim D}{\mathbb{E}} [\operatorname{Loss}(y)]\right).
\]

The expected loss \( \underset{y \sim D_i}{\mathbb{E}} [\operatorname{Loss}(y)] \) can be approximated using the Data Mixing Law:
\[
\underset{y \sim D_i}{\mathbb{E}} [\operatorname{Loss}(y)] = L_{D_i}(\alpha_1, \ldots, \alpha_m) \approx c_i + k_i \exp \left(\sum_{j=1}^M t_{ij} \alpha_j \right).
\]

Given the generated data proportions \( \gamma_1, \ldots, \gamma_M \) and applying the Data Mixing Law, we have the approximation:
\[
exp\left(-c_i - k_i \exp \left(\sum_{j=1}^M t_{ij} \alpha_j \right)\right) \approx exp\left(-\underset{y \sim D}{\mathbb{E}} [\operatorname{Loss}(y)]\right) \approx \gamma_i.
\]

Assuming the accuracy of the Data Mixing Law and a large sample size, we can simplify this to:
\[
\gamma_i = exp\left({-\left(c_i + k_i \exp \left(\sum_{j=1}^M t_{ij} \alpha_j \right)\right)}\right).
\]

Taking the logarithm on both sides, we obtain:
\[
-c_i - k_i \exp \left(\sum_{j=1}^M t_{ij} \alpha_j \right) = \log \gamma_i.
\]

This can be further simplified to:
\[
\exp \left(\sum_{j=1}^M t_{ij} \alpha_j \right) = \frac{-(\log \gamma_i + c_i)}{k_i}.
\]

Taking the logarithm again on both sides, we derive:
\[
\sum_{j=1}^M t_{ij} \alpha_j = \log \left(-\frac{\log \gamma_i + c_i}{k_i}\right).
\]

This expression can be written in matrix form as:
\[
T \vec{\alpha} = \vec{\beta},
\]
where \( \beta_i = \log \left(-\frac{\log \gamma_i + c_i}{k_i}\right) \) and \( T_{ij} = t_{ij} \).

Finally, the training data proportions \( \vec{\alpha} \) can be calculated by inverting the matrix \( T \) and applying it to \( \vec{\beta} \):
\[
\vec{\alpha} = T^{-1} \vec{\beta}.
\]
This completes the proof.
\end{proof}

Proposition \ref{pro:2} demonstrates the connection between synthetic data proportions and pre-training data proportions using the data mixing law. However, this proof provides only a theoretical guarantee, and the method for calculating and estimating the pre-training data proportions remains unclear. A practical algorithm for this estimation is summarized in Section \ref{sec: Preliminary Algorithm}.

\subsection{A Preliminary Algorithm for Data Proportion Detection}\label{sec: Preliminary Algorithm}
Based on the proof in section \ref{sec: Theory_Data_Detection}, we conduct a practical algorithms for data proportion detection. The data proportion detection algorithm consists of three key steps:
\begin{algorithm}[t]
  \caption{data proportion detection Process}\label{algorithm:1}
  \KwIn{Number of samples $M$, Parameters $t_{ij}$, $k_i$, $c_i$ from Data Mixing Law and LLM}
  \KwOut{Estimated mixing coefficients $\hat{\alpha}_1, \hat{\alpha}_2, \dots, \hat{\alpha}_n$}
  
  \tcp{Step 1: Data Generation}
  Generate $M$ samples $\{y_i\}_{i=1}^M$, where $\{y_i\}=LLM(<bos>)$;

  \tcp{Step 2: Classification}
  Classify the $M$ samples into $n$ categories using a classification model, resulting in datasets $d_1, d_2, \dots, d_n$. Estimate the proportion $\gamma_1, \gamma_2, \dots, \gamma_n$ for each category as follows:
  \[
    \gamma_i \approx \frac{1}{M} \sum_{j=1}^M \mathbb{I}(y_j \in d_i), \quad y_j \sim d_i \ \text{i.i.d.}
  \]

  \tcp{Step 3: Pre-training Proportion Estimation}
  Estimate the mixing coefficients $\hat{\alpha}_1, \hat{\alpha}_2, \dots, \hat{\alpha}_n$ based on the estimated perplexities $\gamma_i$ and the Data Mixing Law using the formula:
  \[
  \vec{\hat{\alpha}} = T^{-1} \vec{\beta},
  \]
  where \( \beta_i = \log \left(-\frac{\log \gamma_i + c_i}{k_i}\right) \) and \( T_{ij} = t_{ij} \).

\Return $\hat{\alpha}_1, \hat{\alpha}_2, \dots, \hat{\alpha}_n$
\end{algorithm}

\textbf{Step 1: Data Generation} For the LLM, generate $M$ samples $\{y_i\}_{i=1}^M$ by sampling from $LLM(<bos>)$ with a temperature of 1.0 and without top-k filtering. The value of $M$ is typically large to approximate the target perplexity accurately.

\textbf{Step 2: Classification and Perplexity Approximation} We first train a classification model for classifying the models to domain $D_1, D_2,..., D_n$, and the classified data is denoted as $d_1, d_2, ..., d_n$. For each category we approximate the perplexity follow the following formula:
\[
    \gamma_i \approx \frac{1}{M} \sum_{j=1}^M \mathbb{I}(y_j \in d_i), \quad y_j \sim d_i
\]

\textbf{Step 3: Calculate data proportion detection Proportion} In the final step, we estimate the data proportion of pre-training data following the following formula:
  \[
  \vec{\hat{\alpha}} = T^{-1} \vec{\beta},
  \]
  where \( \beta_i = \log \left(-\frac{\log \gamma_i + c_i}{k_i}\right) \) and \( T_{ij} = t_{ij} \).

The algorithm's key steps are concisely summarized in Algorithm \ref{algorithm:1}. Through these steps, we conduct preliminary experiments in section \ref{sec: Preliminary Experiments}. The experiment reveals several challenges, including fast LLM inference, robust data cleaning and classification
systems and improving the next generation of data mixing laws. Addressing these challenges is crucial for enhancing the performance and reliability of data proportion detection in practical applications. We will discuss these challenges in detail in section \ref{sec: Challenges}.



\section{Preliminary Experiments}\label{sec: Preliminary Experiments}
In this section, we conduct preliminary experiments and analyze the challenges faced in data proportion detection. Since the data mixing law parameters are not available for LLMs, we use the first two steps in algorithm \ref{algorithm:1} of synthetic data to estimate the data proportion of training data.

\subsection{Experiment Setting}
\paragraph{Base Model} We choose MAP-NEO 7B Base, a state-of-the-art open-source model with a well-documented pre-training data distribution. This model serves as a benchmark for evaluating the effectiveness of data proportion detection algorithms.

\paragraph{Dataset} We generate 100K of data from <bos> using MAP-NEO 7B Base. The 100K data were then classified using the classification model for estimation of the proportion of each category.

\paragraph{Classification Model} We fine-tune Meta-Llama-3-8B-Instruct as our classification model. By leveraging 10000 pre-training data from each category in MAP-NEO for fine-tuning, we achieve a classification accuracy exceeding 90\%.

\paragraph{Setting}
We set the temperature to 1.0 for random generation, without applying top\_k inference. All experiments are conducted on a machine equipped with 8 NVIDIA A100 GPUs, a 120-core CPU, and 960GB of memory.

\begin{figure*}
  \centering
  \begin{minipage}[b]{0.47\textwidth}
    \includegraphics[width=\textwidth]{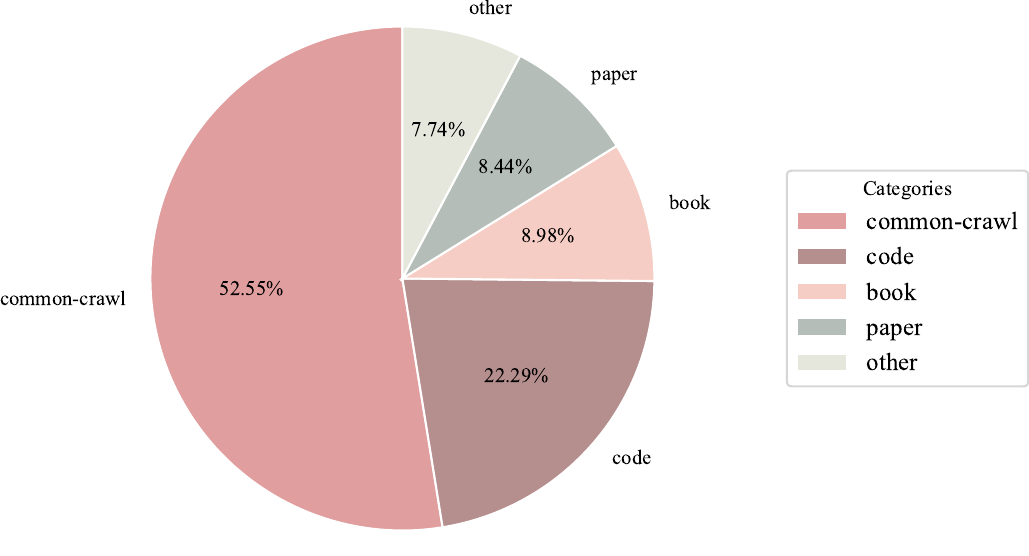}
    \caption{Training data proportion}
    \label{fig: Distribution_ori}
  \end{minipage}
  \hfill
  \begin{minipage}[b]{0.47\textwidth}
    \includegraphics[width=\textwidth]{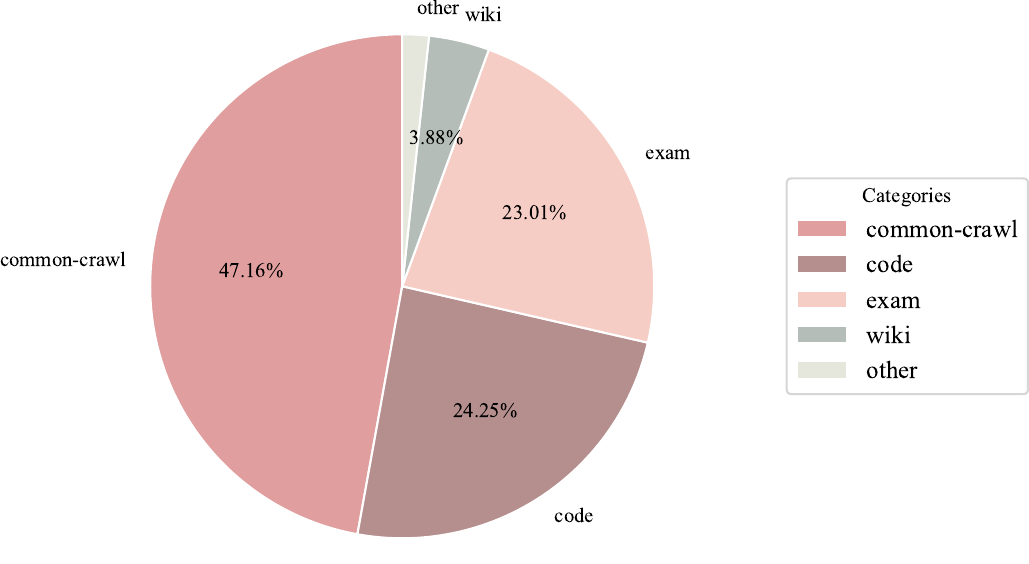}
    \caption{Detect data proportion}
    \label{fig: Distribution_detect}
  \end{minipage}
\end{figure*}

\subsection{Experimental Results}

We generated a synthetic dataset of 100k examples to assess the proportion of synthetic data. A classification model was employed to automatically categorize the data. Since no formal data mixing guidelines for MAP-NEO are available, we rely on the common assumption that an increase in data from a specific domain results in a lower loss for that domain. Consequently, we use the proportion of generated data to estimate the actual training data distribution.

As illustrated in Figure \ref{fig: Distribution_ori}, we present the actual training data distribution. Figure \ref{fig: Distribution_detect} shows the estimated data proportions. The initial data proportion detection process accurately identifies the proportions of common-crawl and code, as they constitute a significant portion of the dataset. However, for other categories, the preliminary detection method fails to accurately estimate the data proportions. We attribute this discrepancy to the lack of a comprehensive data cleaning system and the absence of a robust data classification framework.

\section{Challenges \& Future Directions}\label{sec: Challenges}

Given the potential inaccuracies in data proportion detection, we outline the challenges and future directions in this area, giving insights of future directions in the data proportion detection topic.

\subsection{Fast and Large-Scale LLM Inference Systems}
Following Algorithm \ref{algorithm:1}, we need large amount of generated data from <bos> to estimate the generated data proportion accurately. To support the large amount of generated data, there is a need to develop fast LLM inference frameworks.  Currently, although significant advancements like KV-cache have been made in accelerating LLMs inference~\cite{kwon2023efficient, zhang2024pqcache}, further improvements in both algorithms and hardware support are essential. Thus, we must prioritize the development of faster, large-scale inference-supported frameworks. 

\subsection{Robust Data Cleaning and Classification Systems}
\paragraph{Synthetic Data Cleaning System}
Synthetic data can be highly unstructured, particularly for smaller models. Models without supervised fine-tuning often produce inconsistent outputs, as shown in Figure \ref{fig: Cleaning}. Although some pioneering works address the cleaning of low-quality web text~\cite{wenzek2020ccnet}, these approaches are insufficient for handling generated data. As a result, specialized data cleaning techniques tailored to generated data are necessary. One possible solution is to leverage LLMs to clean and rephrase the generated sentences.

\paragraph{Robust Data Classification System} After data cleaning, it is crucial to ensure accurate model classification. While our fine-tuned system achieves over 90\% accuracy, surpassing 95\% remains a challenge for LLMs, indicating that current classification methods are still inadequate. Additionally, the system lacks robustness, as varying models produce diverse data types, leading to inconsistent classification performance. Therefore, developing a more reliable and resilient data classification system is essential. One possible solution can be preparing large amount of high-quality classification data to fine-tune LLMs and utilize prompt engineering for better performance.

\begin{figure}[ht]
  \includegraphics[width=0.47\textwidth]{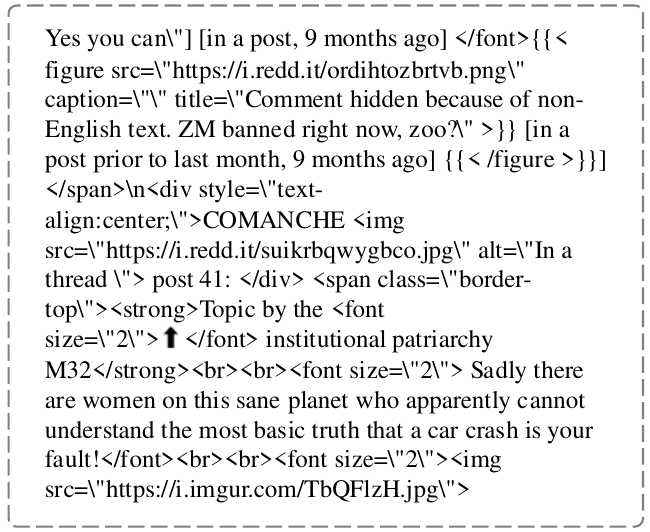}
  \caption{MAP-NEO synthetic data from <bos>. The data can be low quality and propose challenges for classification models. Hence we need strong classification models and data cleaning.}
  \label{fig: Cleaning}
  \vspace{-3mm}
\end{figure}

\subsection{Next Generation of Data Mixing Laws}
\paragraph{Data Mixing Laws for LLMs}
It is generally understood that increasing the amount of training data in a specific domain tends to reduce loss in that domain. However, the precise impact of training data proportions on loss remains unclear. To the best of our knowledge, the data mixing law is the only framework that mathematically describes this relationship.

Although \cite{ye2024data} has proposed scaling laws, they have not conducted experiments specifically for LLMs. More accurate scaling laws tailored to LLMs are still required.

\paragraph{Efficient Estimation of Data Mixing Laws} Estimating data mixing laws across $n$ domains typically requires extensive LLM pre-training, leading to significant computational overhead. To mitigate this inefficiency, we propose a more streamlined approach for estimating data mixing laws. In this paper, we introduce a preliminary method for estimating the data mixing law using identical projection. However, this approach is not sufficiently accurate because the domains are not identical. Therefore, there is a need for a more efficient, effective, and robust data mixing law to guide future work.

\subsection{Robust Data Preparing System}
After detecting the data proportions, we need to prepare an enormous volume of pre-training data, typically exceeding 30 TB of raw tokens~\cite{together2023redpajama}. Common data processing methods, such as those described in \cite{wenzek2020ccnet}, involve filtering, deduplication, language identification, and language model-based filtering. This necessitates the development of a robust data management system capable of handling over 30 TB of data. Although Data Juicer has been a pioneering system for data preparation, it still faces challenges when processing such large-scale datasets. Therefore, both algorithmic and hardware improvements are essential to further enhance data preparation systems.

In summary, the challenges in data proportion detection for LLM pre-training are multifaceted and require significant advancements across several domains. Data Proportion detection requires both advanced model and data management techniques. Addressing these challenges will not only improve the performance of LLMs but also enhance the understanding of optimal data management practices for pre-training.


\section{Conclusion}\label{sec: Conclusion}
In this paper, we proposed a new topic called data proportion detection which can automatically detect the pre-training data proportion from LLMs. We rigorously define this problem and then gave proof connecting generated data with pre-training data proportion. We then proposed practical algorithms for estimate pre-training data proportions. Since data mixing laws for LLMs are still not available, we gave preliminary estimation using the proportion of synthetic data. After that, we gave valuable insights and possible future works for data proportion detection.

\newpage
\bibliographystyle{ACM-Reference-Format}
\bibliography{sample}

\end{document}